\documentclass{article}
\usepackage{ijcai16}
\usepackage{amsmath}
\usepackage{amsthm}
\usepackage{clrscode3e}

\usepackage{times}
\newtheoremstyle{example}{\topsep}{\topsep}%
{}%         Body font
{}%         Indent amount (empty = no indent, \parindent = para indent)
{\bfseries}% Thm head font
{}%        Punctuation after thm head
{\newline}%     Space after thm head (\newline = linebreak)
{\thmname{#1}\thmnumber{ #2}\thmnote{ #3}}%         Thm head spec

\theoremstyle{example}

\title{Probabilistic Inference Modulo Theories\thanks{This StarAI-16 paper is a very close revision of [de Salvo Braz \emph{et al.}, 2016]. }}
\author{
	Rodrigo de Salvo Braz\\
	SRI International\\
	Menlo Park, CA, USA
	\And
	Ciaran O'Reilly\\
	SRI International\\
	Menlo Park, CA, USA
	\And
	Vibhav Gogate\\
	U. of Texas at Dallas\\
	Dallas, TX, USA
	\And
	Rina Dechter\\
	U. of California, Irvine\\
	Irvine, CA, USA
}
\date{\today}

\usepackage{amsmath}
\usepackage{amssymb}
\usepackage{dsfont}
\usepackage{xspace}
\usepackage{algorithm}
\usepackage{graphicx}

\usepackage{mathtools}

\DeclarePairedDelimiter{\floor}{\lfloor}{\rfloor}

\newtheorem{theorem}{Theorem}[section]

\renewcommand{\gets}{\leftarrow}

\newcommand{\algo}{SGDPLL($T$)\xspace}
\newcommand{\sgve}{SGVE($T$)\xspace}

\newcommand{\ifte} [3]{\ensuremath{\mathtt{if}\,#1\,\mathtt{\;then}\,#2\,\mathtt{\;else}\,#3}\xspace}

\newcommand{\iftt}  {\mathtt{if }\;}
\newcommand{\thentt}{\;\mathtt{ then }\;}
\newcommand{\elsett}{\;\mathtt{ else }\;}

\newcommand{\true}{\ensuremath{\textsc{true}}\xspace}
\newcommand{\false}{\ensuremath{\textsc{false}}\xspace}

\renewcommand{\L}{\ensuremath{\mathcal{L}}\xspace}
\newcommand{\C}{\ensuremath{\mathcal{C}}\xspace}

\newcommand{\TL}{\ensuremath{T_{\L}}\xspace}
\newcommand{\TC}{\ensuremath{T_{\C}}\xspace}

\newcommand{\by}{\ensuremath{\mathbf{y}}\xspace}

\newcommand{\T}{\ensuremath{T}\xspace}

\newcommand{\define}[1]{\textbf{#1}}

\newcommand{\mycaption}[1]{\caption{\small{#1}}}
\newcommand{\overbar}[1]{\mkern 1.5mu\overline{\mkern-1.5mu#1\mkern-1.5mu}\mkern 1.5mu}

\newcommand{\eat}[1]{}

\eat{% some tricks added by Vibhav
	\setlength{\tabcolsep}{3pt}
	\addtolength{\textfloatsep}{-2mm}
	\addtolength{\floatsep}{-1mm}
	\addtolength{\abovecaptionskip}{-2mm}
	\addtolength{\belowcaptionskip}{-2mm}
	\addtolength{\abovedisplayskip}{-1mm}
	\addtolength{\belowdisplayskip}{-1mm}
	\addtolength{\arraycolsep}{-1mm}
	\addtolength{\topsep}{-1mm}
	\addtolength{\partopsep}{-0.5mm}
	\addtolength{\itemsep}{-1mm}
	\addtolength{\textwidth}{4mm}
	\addtolength{\textwidth}{2mm}

	\frenchspacing
} % hiding tricks

\begin{document}
	
\maketitle

\begin{abstract}
	We present \algo, an algorithm that solves (among many other problems) probabilistic inference modulo theories, that is, inference problems over probabilistic models defined via a logic theory provided as a parameter (currently, propositional, equalities on discrete sorts, and inequalities, more specifically difference arithmetic, on bounded integers).
	While many solutions to probabilistic inference over logic representations have been proposed,
	\algo is simultaneously (1) lifted, (2) exact and (3) modulo theories, that is, parameterized by a background logic theory.
	This offers a foundation for extending it to rich logic languages such as data structures and relational data.
	By lifted, we mean algorithms with constant complexity in the domain size (the number of values that variables can take). We also detail a solver for summations with difference arithmetic and show experimental results from a scenario in which \algo is much faster than a state-of-the-art probabilistic solver.
\end{abstract}

\section{Introduction}

\nocite{desalvobraz2016probabilistic} % since it is cited on the title footnote which does not allow \cite to be used
High-level, general-purpose uncertainty representations as well as fast inference and learning for them are important goals in Artificial Intelligence.
In the past few decades, graphical models have made tremendous progress 
towards achieving these goals, but even today their main methods can only support very
simple types of representations such as tables and weight matrices that exclude logical constructs
such as relations, functions, arithmetic, lists, and trees.
For example, consider the following conditional probability distributions,
which would need to be either automatically expanded into large tables or, at best, decision diagrams (a process called \emph{propositionalization}),
or manipulated in a manual, ad hoc manner,
in order to be processed by mainstream probabilistic inference algorithms from the graphical models literature:
\begin{itemize}
	\item $P(x > 10 \,|\, y \neq 98 \vee z \leq 15) = 0.1$,\\ for $x,y,z \in \{1,\dots,1000\}$
	\item $P(x \neq \id{Bob} \,|\, \id{friends(x,\id{Ann})}) = 0.3$
\end{itemize}	
Early work in Statistical Relational Learning \cite{getoor&taskar07} offered more expressive languages that used relational logic to specify probabilistic models but relied on conversion to conventional representations to perform inference,
which can be very inefficient.
To address this problem, lifted probabilistic inference  algorithms \cite{poole03,braz07,gogate&domingos11b,broeck&al11}
were proposed for efficiently processing logically specified models at the abstract first-order level. However, even these algorithms can only handle languages having limited expressive power (e.g., function-free first-order logic formulas).
More recently,
several probabilistic programming languages \cite{goodman12church} have been proposed that enable probability distributions to be specified using high-level programming languages (e.g., Scheme). However, the state-of-the-art of inference over these languages is essentially approximate inference methods that operate over a propositional (grounded) representation.

We present \algo, an algorithm that solves (among many other problems) probabilistic inference on models defined over higher-order logical representations.
Importantly, the algorithm is agnostic with respect to which particular logic theory is used,
which is provided to it as a parameter.
We have so far developed solvers for propositional, equalities on categorical sorts, and inequalities, more specifically difference arithmetic, on bounded integers (only the latter is detailed in this paper, as an example).
However, \algo offers a foundation for extending it to richer theories involving relations, arithmetic, lists and trees.
While many algorithms for probabilistic inference over logic representations
have been proposed, \algo is simultaneously
(1) lifted,
(2) exact\footnote{Our emphasis on exact inference, which is impractical for most real-world problems, is due to the fact that it is a needed basis for flexible and well-understood approximations (e.g., Rao-Blackwellised sampling).}
and
(3) modulo theories.
By lifted, we mean algorithms with constant complexity in the domain size (the number of values that variables can take).

\algo generalizes the  
Davis-Putnam-Logemann-Loveland (DPLL) algorithm for solving the satisfiability problem in the following ways:
(1) while DPLL only works on propositional logic,
\algo takes (as mentioned) a logic theory as a parameter;
(2) it solves many more problems than satisfiability on boolean formulas,
including summations over real-typed expressions,
and (3) it is \emph{symbolic}, accepting input with free variables
(which can be seen as constants with unknown values)
in terms of which the output is expressed.

Generalization (1) is similar to the generalization of DPLL
made by Satisfiability Modulo Theories (SMT) \cite{barrett&al09,moura&al07,ganzinger04dpllt},
but SMT algorithms require only satisfiability solvers of their theory parameter
to be provided, whereas \algo may require solvers for harder tasks (including model counting).
Figures \ref{fig:dpll} and \ref{fig:sgdpllt-example}
illustrate how both DPLL and \algo work
and highlight their similarities and differences.

Note that \algo is not a \emph{probabilistic} inference algorithm in a \emph{direct} sense,
because its inputs are not defined as probability distributions, random variables, or
any other concepts from probability theory.
Instead, it is an \emph{algebraic} algorithm defined in terms of expressions, functions, and quantifiers.
However, probabilistic inference on rich languages can be reduced to 
tasks that \algo can efficiently solve, as shown in Section \ref{sec:probabilistic-inference}.

The rest of this paper is organized as follows:
Section \ref{sec:dpll-smt-and-sgdpllt} describes how \algo generalizes DPLL and SMT algorithms
Section \ref{sec:t-problems-and-t-solutions} defines $T$-problems and $T$-solutions, Section \ref{sec:sgdpllt} describes \algo that solves $T$-problems,
Section \ref{sec:probabilistic-inference} explains how to use \algo to solve probabilistic inference modulo theories, Section \ref{sec:experiment} describes a proof-of-concept experiment comparing our solution to a state-of-the-art probabilistic solver, Section \ref{sec:related-work} discusses related work, and Section \ref{sec:conclusion} concludes. A specific solver for summation over difference arithmetic and polynomials is described in Appendices \ref{sec:difference-arithmetic-base-case} and \ref{sec:faulhaber}.

\section{DPLL, SMT and \algo}
\label{sec:dpll-smt-and-sgdpllt}

\begin{figure}[t]
	\centerline{\includegraphics[scale=0.7]{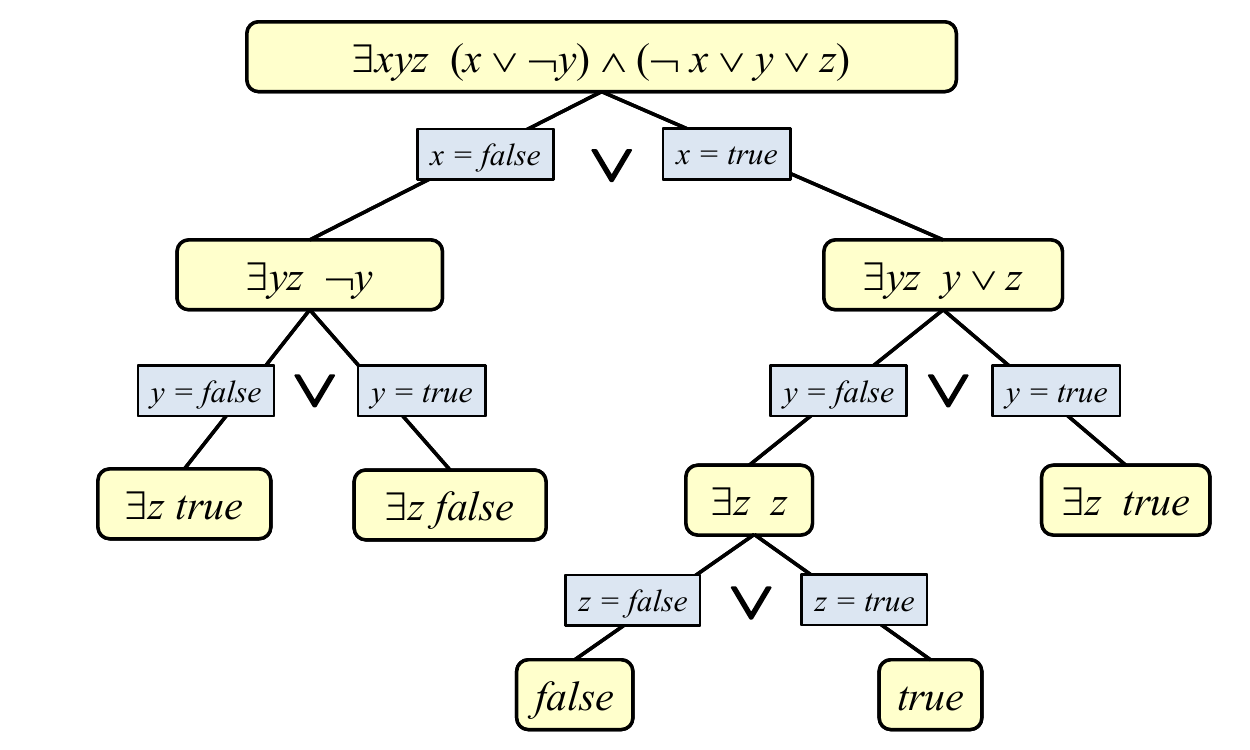}}
	\mycaption{Example of DPLL's search tree for the existence of satisfying assignments. We show the full tree even though the search typically stops when the first satisfying assignment is found.}
	\label{fig:dpll}
\end{figure}

\begin{figure}[t]
	\centerline{\includegraphics[scale=0.7]{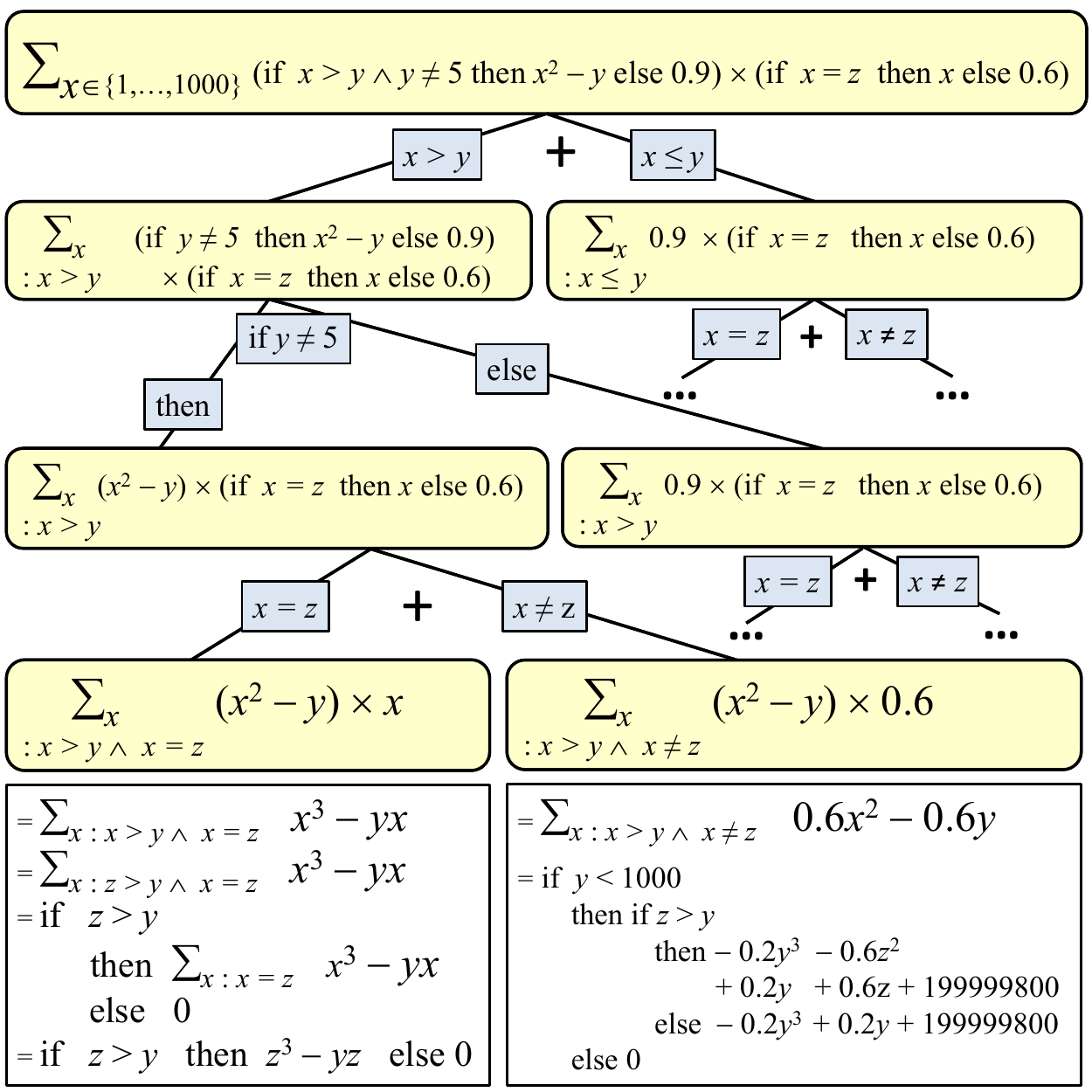}}
	\mycaption{
		\algo for summation with a background theory of difference arithmetic on bounded integers.
		Variables $x,y,z$ are in $\{1,\dots,1000\}$ but \algo does not iterate over all these values.
		It splits the problem according to literals in the background theory, simplifying it until 
		the sum is over a literal-free expression (here, polynomials).
		Splits on literals in the quantified variable $x$ split its quantifier $\sum$
		and the solutions to the sub-problems are combined by $+$
		(\emph{quantifier-splitting} as explained in Section \ref{sec:sgdpllt}).
		The split on $y \neq 5$ does not involve index $x$,
		so it creates an if-then-else expression (\emph{if-splitting}).
		Literal $y\neq 5$ (and its negation) does not need to be in the sub-solutions,
		from which it is simplified away;
		it will be present in the final solution only, as the if-then-else condition.
		When the base case with a literal-free expression is obtained,
		the specific theory solver computes its solution
		as detailed in the Appendices (lower rectangular boxes).
		The figure omits the simplification
		of the overall resulting expression by summation of sub-solutions and possible
		elimination of redundant literals.
		Problems with multiple $\sum$ quantifiers are solved by successively solving the innermost one,
		treating the indices of external sums as free variables.
	}
	\label{fig:sgdpllt-example}
\end{figure}

The \define{Davis-Putnam-Logemann-Loveland (DPLL)} algorithm \cite{davis&al62} solves the \define{satisfiability} (or \define{SAT}) problem. SAT consists of determining whether a propositional formula $F$, expressed in conjunctive normal form (CNF), has a solution or not. A CNF is a conjunction ($\wedge$) of clauses where a clause is a disjunction ($\vee$) of literals. A literal is either a proposition (that is, a Boolean variable) or its negation. A solution to a CNF is an assignment of values from the set $\{\true,\false\}$ to all propositions in $F$ such that at least one literal in each clause in $F$ is assigned to \true.

\begin{footnotesize}
	\begin{algorithm}[h]
		
		\begin{codebox}
			\Procname{$\proc{DPLL}(F)$}
			\zi $F$: a formula in CNF.
			%\zi $\id{simplify}$: simplifies boolean formulas given a condition
			%\zi (e.g., $\id{simplify}(x \wedge y | \neg y) = \textsc{false}$)
			%\zi
			\li \If {$F$ is a boolean constant}
			\li 	\Then
			\Return $F$
			\li 	\Else
			$v \gets $ pick a variable in $F$
			\li 		$\id{Sol}_1 \gets \id{DPLL}(\id{simplify}(F | \,\,\,\, v))$
			\li 		$\id{Sol}_2 \gets \id{DPLL}(\id{simplify}(F | \neg  v))$
			\li 		\Return $\id{Sol}_1 \vee \id{Sol}_2$
			\End
		\end{codebox}
		\mycaption{A version of the DPLL algorithm.}
		\label{alg:dpll}
	\end{algorithm}
\end{footnotesize}
Algorithm \ref{alg:dpll} shows a simplified, non-optimized version of DPLL
which operates on CNF formulas. It works by recursively trying assignments for each proposition, 
one at a time, simplifying the CNF, until $F$ is a constant (\true or \false),
and combining the results with disjunction.
Figure \ref{fig:dpll} shows an example of the execution of DPLL. DPLL is the basis for modern SAT solvers which improve it by adding sophisticated techniques such as unit propagation, watch literals, and clause learning \cite{een&sorenson03,maric09formalization}.

\define{Satisfiability Modulo Theories (SMT)} algorithms \cite{barrett&al09,moura&al07,ganzinger04dpllt} generalize DPLL and can determine the satisfiability of a Boolean formula expressed in first-order logic, where some function and predicate symbols have specific interpretations. Examples of predicates include equalities, inequalities, and uninterpreted functions, which can then be evaluated using rules of real arithmetic. SMT algorithms condition on the \emph{literals} of a background theory $T$, looking for a truth assignment to these literals that satisfies the formula.
While a SAT solver is free to condition on a proposition, assigning it to either \textsc{true} or \textsc{false}
regardless of previous choices (truth values of propositions are independent from each other),
an SMT solver needs to also check whether a choice for one literal is \emph{consistent}
with the previous choices for others, according to $T$. This is done by a theory-specific model 
checker, provided as a parameter.

\define{\algo} is, like SMT algorithms, \define{modulo theories} but further generalizes DPLL by
being \define{symbolic} and \define{quantifier-parametric} (thus ``Symbolic Generalized DPLL($T$)'').
These three features can be observed in the problem being solved by \algo in Figure \ref{fig:sgdpllt-example}:
\begin{align*}
	\sum_{x \in \{1,\dots,1000\}} & \quad  (\ifte{x > y \wedge y \neq 5}{x^2-y}{0.9}) \\
	& \quad \times (\ifte{x = z}{x}{0.6})
\end{align*}
In this example, the problem being solved
requires more than propositional logic theory since equality, inequality
and other functions are involved.
The problem's quantifier is a summation,
as opposed to DPLL and SMT's existential quantification $\exists$. Also, the output will be \emph{symbolic} in $y$ and $z$ because these variables are not being quantified,
as opposed to DPLL and SMT algorithms which implicitly assume all variables to be quantified.

Before formally describing \algo, we will further comment on its three key generalizations.\\
\textbf{1. Quantifier-parametric.}
Satisfiability can be seen as computing the value of an existentially quantified formula;
the existential quantifier can be seen as an indexed form of disjunction,
so we say it is \define{based} on disjunction.
\algo generalizes SMT algorithms 
by solving \define{\emph{any} quantifier $\bigoplus$
based on a commutative associative operation $\oplus$},
provided that a corresponding theory-specific solver is available for base case problems,
as explained later.
Examples of ($\bigoplus$, $\oplus$, ) pairs are ($\forall$,$\wedge$),
($\exists$,$\vee$), ($\sum$,$+$), and ($\prod$,$\times$).
Therefore \algo can solve not only satisfiability (since disjunction is commutative and associative),
but also validity (using the $\forall$ quantifier), sums, products, model counting, weighted model counting, maximization, among others, for propositional logic-based, and many other, theories.\\
\textbf{2. Modulo Theories.}
SMT generalizes the propositions in SAT to literals in a given theory $T$,
but the theory connecting these literals remains that of boolean connectives.
\algo takes a theory $T=(\TC,\TL)$, composed of a \define{constraint theory} \TC and an \define{input theory} \TL. DPLL propositions are generalized to literals in \TC in \algo,
whereas the boolean connectives are generalized to functions in \TL.
In the example above, \TC is the theory of difference arithmetic on bounded integers,
whereas \TL is the theory of $+, \times$, boolean connectives and $\ifte{}{}{}$.
Of the two, \TC is the crucial one, on which inference is performed,
while
\TL is used simply for the simplifications after conditioning,
which takes time at most linear in the input expression size.\\
\textbf{3. Symbolic.}
Both SAT and SMT can be seen as computing the value of an existentially quantified formula
in which \emph{all} variables are quantified, and which is always equivalent to either \textsc{true} or \textsc{false}.
\algo further generalizes SAT and SMT by accepting quantifications over \emph{any subset} of the variables in its input expression (including the empty set).
The non-quantified variables are \define{free variables},
and the result of the quantification will typically depend on them.
Therefore, \algo's output is
a \define{symbolic} expression in terms of free variables.
Section \ref{sec:t-problems-and-t-solutions} shows an example of a symbolic solution.

Being symbolic allows \algo to conveniently solve a number of problems, including quantifier elimination and exploitation of factorization in probabilistic inference, as discussed in Section \ref{sec:probabilistic-inference}.

\section{$T$-Problems and $T$-Solutions}
\label{sec:t-problems-and-t-solutions}

\algo receives a \define{$T$-problem} (or, for short, a \define{problem}) of the form
\begin{equation}
	\bigoplus_{x:F(x,\by)} E(x,\by),
	\label{eqn:t-problem}
\end{equation}
where $x$ is an \define{index variable} quantified by $\bigoplus$ and
subject to constraint $F(x,\by)$ in \TC, with possibly the presence of \define{free variables} \by,
and $E(x,\by)$ an expression in \TL.
$F(x,\by)$ is a conjunction of literals in \TC, that is, a conjunctive clause.
An example of a problem is
\[\sum_{x:3 \leq x \wedge x \leq y} \ifte{x > 4}{y}{10 + z},\]
for $x, y, z$ bounded integer variables in, say, $\{1,\dots,20\}$.
The index is $x$ whereas $y,z$ are free variables.

A \define{$T$-solution} (or, for short, simply a \define{solution}) to a problem is
simply a quantifier-free expression in \TL equivalent to the problem.
Note that solution will often contain literals and conditional expressions
dependent on the free variables.
%
%It can be an \define{unconditional solution}, containing no literals in \TC,
%or a \define{conditional solution} of the form \ifte{L}{S_1}{S_2}, where $L$ is a literal in \TC
%and $S_1,S_2$ are two solutions (either conditional or unconditional).
%Note that, in order for the solution to be equivalent to the problem, only
%variables that were free (not quantified) can appear in the literal $L$.
%In other words, a solution can be seen as a decision tree on literals, with literal-free
%expressions in the leaves, such that each leaf is equivalent to the original problem,
%provided that the literals on the path to it are true.
%
For example, the problem $$\sum_{x:1 \leq x \wedge x \leq 10} \ifte{y > 2 \wedge w > y}{y}{4}$$
has an equivalent conditional solution $$\ifte{y > 2}{\ifte{w > y}{10y}{40}}{40}.$$

For more general problems with \define{multiple quantifiers}, we simply successively solve the innermost problem
until all quantifiers have been eliminated.

\section{\algo}
\label{sec:sgdpllt}
In this section we provide the details of \algo,
described in Algorithm \ref{alg:sgdpllt}
and exemplified in Figure \ref{fig:sgdpllt-example}.

\subsection{Solving Base Case $T$-Problems}

A problem, as defined in Equation (\ref{eqn:t-problem}), is in \define{base case} if $E(x,\by)$ contains no literals in \TC.

In this paper, $\T=(\TC,\TL)$ where \TL is polynomials over bounded integer variables,
and \TC is \define{difference arithmetic} \cite{moura&al07},
with atoms of the form $x < y$ or $x \leq y + c$, where $c$ is an integer constant.
Strict inequalities $x < y + c$ can be represented as $x \leq y + c - 1$
and the negation of $x \leq y + c$ is $y \leq x -c - 1$.
From now on, we shorten $a \leq x \wedge x \leq b$ to $a \leq x \leq b$.

Therefore, a base case problem for this theory is of the form $\sum_{x:F(x,\by)} P(x,\by)$,
where $x$ is the index, $\by$ is a tuple of free variables,
$F(x,\by)$ is a conjunction of difference arithmetic literals,
and $P(x,\by)$ is a polynomial over $x$ and \by.
We show how to fully solve difference arithmetic base cases in Appendices \ref{sec:difference-arithmetic-base-case} and \ref{sec:faulhaber}.

\subsection{Solving Non-Base Case $T$-Problems}

Non-base case problems (that is, those in which $E(x,\by)$ of Equation (\ref{eqn:t-problem}) contains literals in \TC) are solved by reduction to base-case ones.
While base cases are solved by theory-specific solvers,
the reduction from non-base case problems to base case ones is \emph{theory-independent}.
This is significant as it allows \algo to be expanded with new theories
by providing a solver only for base case problems,
analogous to the way SMT solvers require theory solvers only for conjunctive clauses, as opposed to general formulas, in those theories.

The reduction mirrors DPLL, by selecting a \define{splitter literal} $L$ present in $E(x,\by)$ to split
the problem on, generating two simpler problems:
\begin{itemize}
	
	\item \define{quantifier-splitting} applies when $L$ contains the index $x$.
	Then two sub-problems are created, one in which $L$ is added to $F(x,\by)$,
	and another in which $\neg L$ is.
	Their solution is then combined by the quantifier's operation ($+$ for the case of $\sum$).
	
	For example, consider:
	\[\sum_{x:3 < x \leq 10} \ifte{x > 4}{y}{(10 + z)}\]
	To remove the literal from $E(x,\by)$, we add the literal ($x >4$) and its negation ($x \leq 4$) to the constraint on $x$, yielding
	two base-case problems:
	\begin{align*}
		& \Bigl( \sum_{x:x > 4 \wedge 3 < x \leq 10} y \Bigr) + \Bigl( \sum_{x:x \leq 4 \wedge 3 < x \leq 10} (10 + z)\Bigr). \end{align*}
	
	\item \define{if-splitting} applies when $L$ does \emph{not} contain the index $x$.
	Then $L$ becomes the condition of an \ifte{}{}{} expression and the two simpler sub-problems are its \emph{then} and \emph{else} clauses.
	
	For example, consider
	\[\sum_{x:3 < x \leq 10} \ifte{y > 4}{y}{10}.\]
	Splitting on $y > 4$ reduces the problem to
	
	\begin{align*}
		& & \ifte{y > 4}{\sum_{x:3 < x \leq 10} y}{\sum_{x:3 < x \leq 10} 10},
	\end{align*}
	containing two base-case problems.
	
\end{itemize}

%In this example, the two sub-solutions are unconditional polynomials,
%and their sum results in another unconditional polynomial, which is a valid solution.
%However, if at least one of the sub-solutions is conditional,
%their direct sum is not a valid solution.
%In this case, we need to combine them
%with a distributive transformation of $\oplus$ over \ifte{}{}{}:
%\begin{align*}
%& S \oplus (\ifte{L}{S_1}{S_2}) \\
%& \equiv \ifte{L}{S \oplus S_1}{S \oplus S_2},
%\end{align*} 
%proceeding recursively if any of solutions $S, S_2, S_3$ is also conditional.
%For example:
%\begin{align*}
%& (\ifte{x < 4}{y^2}{z}) + (\ifte{y > z}{3}{x}) \\
%& \equiv \iftt x < 4 \thentt y^2 + (\ifte{y > z}{3}{x}) \\ & \hspace{.7in}\elsett z \;\,+ (\ifte{y > z}{3}{x}) \\
%& \equiv \iftt x < 4 \thentt \ifte{y > z}{y^2 + 3}{y^2 + x} \\ & \hspace{.7in}\elsett 
% \ifte{y > z}{z \;\,+ 3}{z \;\,+ x}.
%\end{align*}

The algorithm terminates because each splitting generates sub-problems
with one less literal in $E(x,\by)$, eventually obtaining base case problems.
It is sound because each transformation results in an expression equivalent to the previous one.

To be a valid parameter for \algo, a $(T,\oplus)$-solver $S_T$ for theory $T=(\TL,\TC)$ must,
given a problem $\bigoplus_{x:F(x,\by)} E(x,\by)$,
recognize whether it is in base form and, if so,
provide a solution $base_T(\bigoplus_{x:F(x,\by)} E(x,\by))$.

The algorithm is presented as Algorithm \ref{alg:sgdpllt}.
Note that it does \emph{not} depend on difference arithmetic theory,
but can use a \define{solver for any theory satisfying the requirements above}.

If the $(T,\oplus)$-solver implements the operations above in constant time in the domain size (the size of their types), then it follows that \algo will have complexity \define{independent of the domain size}.
This is the case for the solver for difference arithmetic and will typically be the case for many other solvers.

\begin{algorithm}
	\mycaption{Symbolic Generalized DPLL (\algo),\quad \quad \quad 
		omitting pruning, heuristics and optimizations.}
	\label{alg:sgdpllt}
	
	\begin{codebox}
		\Procname{$\proc{\algo}(\bigoplus_{x:F(x,\by)} E(x,\by))$}
		\zi Returns a $T$-solution for $\bigoplus_{x:F(x,\by)} E(x,\by)$.
		\zi
		\li \If {$E(x,\by)$ is literal-free (base case)}
		\li		\Then \Return $base_T(\bigoplus_{x:F(x,\by)} E(x,\by))$
		\li		\Else 
		\li     	$L \gets $ a literal in $E(x,\by)$
		\li			$E'\;  \gets E$ with $L$ replaced by \true and simplified
		\li			$E'' \gets E$ with $L$ replaced by \false and simplified
		\li			\If {$L$ contains index $x$}
		\li				\Then 
		$\id{Sub}_1 \gets \bigoplus_{x:F(x,\by) \wedge L}     \;\; E'$
		\li					$\id{Sub}_2 \gets \bigoplus_{x:F(x,\by) \wedge \neg L}  E''$
		\li				\Else \Comment $L$ does not contain index $x$:
		\li          		$\id{Sub}_1 \gets \bigoplus_{x:F(x,\by)}  \;E'$
		\li 	     		$\id{Sub}_2 \gets \bigoplus_{x:F(x,\by)}  \;E''$
		\End
		\li 		$S_1 \gets \text{\algo}(Sub_1)$
		\li 		$S_2 \gets \text{\algo}(Sub_2)$
		\li     	\If {$L$ contains index $x$}
		\li         	\Then \Return $S_1 \oplus S_2$
		\li         	\Else \Return the expression $\ifte{L}{S_1}{S_2}$
		\End
		\End
	\end{codebox}
	
\end{algorithm}

\subsection{Optimizations}
\label{sec:optimizations}
In the simple form presented above,
\algo may generate solutions such as $\ifte{x = 3}{\ifte{x \neq 4}{y}{z}}{w}$
in which literals are implied (or negated) by the context they are in,
and are therefore \define{redundant}.
Redundant literals can be eliminated by keeping a conjunction of all choices (sides of literal splittings)
made at any given point (the \define{context})
and using any SMT solver to incrementally decide when
a literal or its negation is implied, thus \define{pruning the search} as soon as possible.
Note that a $(T,\oplus)$-solver for \algo appropriate for $\exists$ can be used for this,
although here there is the opportunity to leverage the very efficient SMT systems already available.

Modern SAT solvers benefit enormously from \define{unit propagation},
\define{watched literals} and \define{clause learning} \cite{een&sorenson03,maric09formalization}.
In DPLL, unit propagation is performed when all but one literal $L$
in a clause are assigned \false.
For this \define{unit clause}, and as a consequence, for the CNF problem, to
be satisfied, $L$ must be \true and is therefore immediately assigned that value wherever it occurs, without the need to split on it.
Detecting unit clauses, however, is expensive if performed
by naively checking all clauses at every splitting.
Watched literals is a data structure scheme that allows only a
small portion of the literals to be checked instead.
Clause learning is based on detecting a subset of jointly unsatisfiable literals
when the splits made so far lead to a contradiction, and keeping it
for detecting contradictions sooner as the search goes on.
In the \algo setting, unit propagation, watched literals and clause learning can be generalized to its
not-necessarily-Boolean expressions;
we leave this presentation for future work.

\section{Probabilistic Inference Modulo Theories}
% with \algo}
\label{sec:probabilistic-inference}

Let $P(X_1 = x_1, \dots, X_n = x_n)$ be the joint probability distribution on random variables $\{X_1, \dots, X_n\}$.
For any tuple of indices $t$, we define $X_t$ to be the tuple of variables indexed by the indices in $t$,
and abbreviate the assignments $(X=x)$ and $(X_t=x_t)$ by simply $x$ and $x_t$, respectively.
Let $\bar{t}$ be the tuple of indices in $\{1,\dots,n\}$ but not in $t$.

The \define{marginal probability distribution} of a subset of variables $X_q$ is
one of the most basic tasks in probabilistic inference, defined as
\begin{align*}
	P(x_q) = \sum_{x_{\bar{q}}} P(x)
\end{align*}
which is a summation on a subset of variables occurring in an input expression,
and therefore solvable by \algo.

If $P(x)$ is expressed
in the language of input and constraint theories appropriate for \algo
(such as the one shown in Figure \ref{fig:sgdpllt-example}), then it can be solved by \algo,
\emph{without} first converting its representation to a much larger one based on tables.
The output will be a summation-free expression in the assignment variables $x_q$
representing the marginal probability distribution of $X_q$.

Let us show how to represent $P(x)$ with an expression in \TL through an example.
Consider a hypothetical generative model involving random variables with bounded integer values and describing the influence of variables such as the number of terror attacks, the Dow Jones index and newly created jobs on the number of people who like an incumbent and an challenger politicians:
\begin{align*}
	& \id{attacks} \sim \id{Uniform}(0..20) \\
	& \id{newJobs} \sim \id{Uniform}(0..100000) \\
	& \id{dow} \sim \id{Uniform}(11000..18000) \\
	& \id{likeChallenger} \sim \id{Uniform}(0..N) \\
	& P(\id{likeIncumbent} \in 0..N | \id{dow}, \id{newJobs}, \id{attacks}) \\
	& =
	\begin{cases}
		\frac{0.4}{\floor{0.7N}}\text{, if }\id{dow} > 16000 \wedge \id{newJobs} > 70000) \\
		\qquad\qquad \wedge \; \id{likeIncumbent} < \floor{0.7N} \\
		\frac{0.6}{N+1-\floor{0.7N}}\text{, if }\id{dow} > 16000 \wedge \id{newJobs} > 70000) \\
		\qquad\qquad\quad\quad\;\;  \wedge \; \id{likeIncumbent} \geq  \floor{0.7N}\\
		\frac{0.8}{\floor{0.5N}}\text{, if }\id{dow} < 13000 \wedge \id{newJobs} < 30000) \\
		\qquad\qquad \wedge \; \id{likeIncumbent} < \floor{0.5N} \\
		\frac{0.2}{N+1-\floor{0.5N}}\text{, if }\id{dow} < 13000 \wedge \id{newJobs} < 30000) \\
		\qquad\qquad\qquad \wedge \; \id{likeIncumbent} \geq  \floor{0.5N}\\
		\frac{0.9}{\floor{0.6N}}\text{, none of the above and }(\id{attacks} \leq 4) \\
		\qquad\qquad \wedge \; \id{likeIncumbent} < \floor{0.6N}\\
		\frac{0.1}{N + 1 -\floor{0.6N}}\text{, none of the above and }(\id{attacks} \leq 4) \\
		\qquad\qquad\qquad \wedge \; \id{likeIncumbent} \geq  \floor{0.6N} \\
		\frac{1}{N+1}\text{, otherwise }
	\end{cases}
\end{align*}
which indicates that, if the Dow Jones index is above 16000 or there were more than 70000 new jobs,
then there is a $0.4$ probability that the number of people who like the incumbent politician
is below around 70\% of $N$ people (and that probability is uniformly distributed among those $\floor{0.7N}$ values), with the remaining $0.6$ probability mass uniformly distributed over the remaining $N+1-\floor{0.7N}$ values.
Similar distributions hold for other conditions.
Note that $N$ is a known parameter and the actual representation will contain the evaluations of its expressions. For example, for $N=10^8$, $0.8/\floor{0.5N}$ is replaced by $1.6 \times 10^{-8}$.

The joint probability distribution
\[P(\id{attacks}, \id{newJobs}, \id{dow}, \id{likeChallenger}, \id{likeIncumbent})\]
is simply the product of
$P(\id{attacks})$, $P(\id{newJobs})$ and so on.
$P(\id{attacks})$
can be expressed by
\[\ifte{\,\id{attacks} \geq 0 \wedge \id{attacks} \leq 20}{1/21}{0}\]
because of its distribution $\id{Uniform}(0..20)$,
and the other uniform distributions are represented analogously.
$P(\id{likeIncumbent}| \id{dow}, \id{newJobs}, \id{attacks})$ is represented by the expression 
\begin{align*}
	&\iftt{\id{dow} > 16000 \wedge \id{newJobs} > 70000} \\
	& \quad \thentt{\iftt{\id{likeIncumbent} < \floor{0.7N}}}\\
	& \quad \quad   \thentt{{\frac{0.4}{\floor{0.7N}}}}\\
	& \quad \quad   \elsett{{\frac{0.6}{N+1-\floor{0.7N}}}}\\
	& \quad \elsett{\iftt{\id{dow} < 13000 \wedge \id{newJobs} < 30000}\; {\dots}}
\end{align*}
again noting that $N$ is fixed and the actual expression contains the constants computed from $\floor{0.7N}$,
$\frac{0.4}{\floor{0.7N}}$, and so on.\footnote{This is due to our polynomial language exclusion of non-constant denominators; \cite{afshar16closedform} describes a piecewise polynomial fraction algorithm that can be the basis of another \algo theory solver allowing this.}

Other probabilistic inference problems can be also solved by \algo.
\define{Belief updating} consists of computing
the \define{posterior probability} of $X_q$
given evidence on $X_e$,
which is defined as
\begin{align*}
	P(x_q | x_e) = \frac{P(x_q, x_e)}{P(x_e)}
	= \frac{P(x_q, x_e)}{\sum_{x_q} P(x_q, x_e)}
	%= \frac{\sum_{x_{\overbar{(q,e)}}} P(x)}{\sum_{x_{\bar{e}}} P(x)}
	%= \frac{P(x_q, x_e)}{\sum_{x_q} P(x_q, x_e)}
\end{align*}
which can be computed with two applications of \algo:
first, we obtain a summation-free expression $S$ for $P(x_q, x_e)$, which is $\sum_{x_{\overbar{(q,e)}}} P(x)$,
and then again $S$ for $\sum_{x_q} P(x_q, x_e)$, which is $\sum_{x_q} S$.

We can also use \algo to compute the \define{most likely assignment} on $X_q$, defined by $\max_{x_q} P(x)$, since $\max$ is a commutative and associative operation.

Applying \algo in the manner above does not take \define{advantage
of factorized representations} of joint probability distributions,
a crucial aspect of efficient probabilistic inference.
However, it can be used as a basis for
an algorithm, \define{Symbolic Generalized Variable Elimination Modulo Theories (\sgve)},
analogous to Variable Elimination (VE) \cite{zhang&poole94,dechter99} for graphical models,
that exploits factorization.
\sgve works in the exact same way VE does, but using \algo whenever VE uses
marginalization over a table.
Note that \algo's symbolic treatment of free variables is crucial for the exploitation of factorization,
since typically only a \emph{subset} of variables is eliminated at each step.
Also note that \sgve, like VE, requires the additive and multiplicative operations
to form a \emph{c-semiring} \cite{bistarelli97semiring}.

Finally, because of \algo and \sgve symbolic capabilities,
it is also possible to compute \define{symbolic query} results as functions of \emph{uninstantiated} evidence variables, without the need to iterate
over all their possible values.\footnote{This concept is also present in \cite{sanner2012symbolic}}
For the election example above with $N=10^8$,
we can compute $P(\id{likeIncumbent} > \id{likeChallenger} | \id{newJobs})$
without providing a value for $\id{newJobs}$, obtaining the symbolic result
\begin{align*}
&\iftt{\id{newJobs} > 70000} \\
&\quad \thentt{0.5173}\\
&\quad \elsett{\iftt{\id{newJobs} < 30000}} \\
&\quad \qquad \qquad  \thentt{0.4316} \\
&\quad \qquad \qquad  \elsett{0.4642}
\end{align*}
without iterating over all values of $\id{newJobs}$.
This result can be seen as a compiled form to be used when the value of $\id{newJob}$ is known,
without the need to reprocess the entire model.

%
%Suppose $P(x)$ is represented as a product of real-valued functions (called \define{factors}) $f_i$:
%\begin{align*}
%P(x) = f_1(x_{t1}) \times \dots \times f_m(x_{t1})
%\end{align*}
%and we want to compute a summation over it:
%\begin{align*}
%\sum_{x_{\bar{q}}} f_1(x_{t1}) \times \dots \times f_m(x_{tm})
%\end{align*}
%where $q$ and $t_i$ are tuples.
%
%We now choose a variable $x_i$ for $i \not\in q$ to eliminate first.
%Let $g$ be the product of all factors in which $x_i$ does not appear,
%$h$ be the product of all factors in which $x_i$ does appear,
%and $b$ be the tuple of indices of variables other than $x_i$ appearing in $h$.
%Then we rewrite the above as
%\begin{align*}
%\sum_{x_{\bar{q},\bar{i}}} g(x_{\bar{i}}) \sum_{x_i} h(x_i,x_b)
%=
%\sum_{x_{\bar{q},\bar{i}}} g(x_{\bar{i}}) h'(x_b)
%\end{align*}
%where $h'$ is a summation-free factor
%computed by \algo and equivalent to
%the innermost summation.
%We now have a problem of the same type as originally,
%but with one less variable,
%and can proceed until all variables in $x_{\bar{q}}$ are eliminated.
%The fact that \algo is symbolic allows us to compute $h'$ without iterating over all
%values to $x_i$.

\section{Experiment}
\label{sec:experiment}

We conduct a proof-of-concept experiment comparing our implementation of \algo-based \sgve (available from the corresponding author's web page)
to the state-of-the-art
probabilistic inference solver variable elimination and conditioning (VEC) \cite{gogate&dechter11},
on the election example described above.
The model is simple enough for \sgve to solve the query
$P(\id{likeIncumbent} > \id{likeChallenger} | \id{newJobs} = 80000 \wedge \id{dow} = 17000)$ exactly
in around 2 seconds on a desktop computer with an Intel E5-2630 processor,
which results in $0.6499$ for $N=10^8$.
The run time of \sgve is constant in $N$;
however, the number of values is too large for a regular solver
such as VEC to solve exactly, because the tables involved will be too large even to instantiate.
By decreasing the range of $\id{newJobs}$ to $0..100$, of $\id{dow}$ to $110..180$ and $N$ to just $500$,
we managed to use VEC but it still takes $51$ seconds to solve the problem.

\section{Related work}
\label{sec:related-work}

\algo is related to many different topics in both
logic and probabilistic inference literature,
besides the strong links to SAT and SMT solvers.

\algo is a lifted inference algorithm \cite{poole03,braz07,gogate&domingos11b},
but lifted algorithms so far have concerned
themselves only with relational formulas with equality.
We have not yet developed the theory solvers
for relational representations required for \algo to do the
same, but we intend to do so using the
already developed modulo-theories mechanism available.
On the other hand, we have presented probabilistic inference over difference arithmetic
for the first time in the lifted inference literature.

\cite{sanner2012symbolic} presents a symbolic variable elimination algorithm (SVE) for hybrid graphical models
described by piecewise polynomials. 
\algo is similar, but explicitly separates the generic
and theory-specific levels, and mirrors the structure of DPLL and SMT.
Moreover, SVE operates on Extended Algebraic Decision Diagrams (XADDs), while \algo operates directly on arbitrary expressions formed with the operators in \TL and \TC.
Finally, in this paper we present a theory solver for sums over bounded integers,
while that paper describes an integration solver for continuous numeric variables
(which can be adapted as an extra theory solver for \algo).
\cite{belle15probabilistic,belle15hashing} extends \cite{sanner2012symbolic}
by also adopting DPLL-style splitting on literals, allowing them to operate directly on general boolean formulas, and by focusing on the use of a SMT solver to prune away unsatisfiable branches.
However, it does not discuss
the symbolic treatment of free variables and its role in factorization,
and does not focus on the generic level (modulo theories) of the algorithm.

\algo generalizes several algorithms that operate on mixed networks \cite{mateescu&dechter08} -- a framework that combines Bayesian networks with constraint networks, but with a much richer representation.
By operating on richer languages, \algo also generalizes exact model counting approaches such as RELSAT \cite{bayardo&pehoushek00} and Cachet \cite{sang&al05b}, as well as weighted model counting algorithms such as ACE \cite{chavira&darwiche08} and formula-based inference \cite{gogate&domingos10}, which use the CNF and weighted CNF representations respectively.

\section{Conclusion and Future Work}
\label{sec:conclusion}

We have presented \algo and its derivation \sgve,
algorithms formally able to solve a variety of problems,
including probabilistic inference modulo theories,
that is, capable of being extended with solvers for
richer representations than propositional logic,
in a lifted and exact manner.

Future work includes additional theories and solvers of interest,
mainly among them algebraic data types and uninterpreted relations;
modern SAT solver optimization techniques such as watched literals, unit propagation and clause learning,
and anytime approximation schemes that offer guaranteed
bounds on approximations that converge to the exact solution.

\section*{Acknowledgments}
We gratefully acknowledge the support of the Defense Advanced Research Projects Agency (DARPA) Probabilistic Programming for Advanced Machine Learning Program under Air Force Research Laboratory (AFRL) prime contracts no. FA8750-14-C-0005 and FA8750-14-C-0011, and NSF grant
IIS-1254071. Any opinions, findings, and conclusions or recommendations expressed in this material are those of the author(s) and do not necessarily reflect the view of DARPA, AFRL, or the US government.

\appendix

\section{Solver for Sum and Difference Arithmetic}
\label{sec:difference-arithmetic-base-case}

This appendix describes a $T$-solver for the base case $T$-problem $\sum_{x:F(x,\by)} P(x,\by)$
for $T=(\TC,\TL)$ where \TC is difference arithmetic and \TL is the language of polynomials,
$x$ is a variable and \by is a tuple of free variables.
Because this is a base case, $P(x,\by)$ is a polynomial and contains no literals.
$F(x,\by)$ is a conjunctive clause of difference arithmetic literals.

The solver also receives, as an extra input, a conjunctive clause $C(\by)$ (a \define{context})
on free variables only,
and its output is a quantifier-free $T$-solution $S(\by)$ such that $C(\by) \Rightarrow S(\by) = \sum_{x:F(x,\by)} P(x,\by)$.
In other words, $C(\by)$ encodes the assignments to \by of interest in a given context,
and the solution needs to be equal to the problem \emph{only} when \by satisfies $C(\by)$.
The context starts with \true but is set to more restrictive formulas in the solver's recursive calls.\footnote{The use of a context here is similar to the one mentioned as an optimization in Section \ref{sec:optimizations}, but while contexts are optional in the main algorithm, it will be seen in the
proof sketch of Theorem \ref{the:difference-arithmetic-solver} that they are required in this solver to ensure termination.}

We assume an SMT (Satisfiability Modulo Theory) solver
that can decide whether a conjunctive clause in the background
theory (here, difference arithmetic) is satisfiable or not.

The intuition behind the solver is gradually removing ambiguities
until we are left with a single lower bound, a single upper bound,
and unique disequalities on index $x$.
For example, if the index $x$ has two lower bounds (two literals $x > y$ and $x > z$),
then we split on $y > z$ to decide which lower bound implies the other,
eliminating it.
Likewise, if there are two literals $x \neq y$ and $x \neq z$,
we split on $y = z$, either eliminating the second one if this is true,
or obtaining a uniqueness guarantee otherwise.
Once we have a single lower bound, single upper bound and unique disequalities,
we can solve the problem more directly, as detailed in Case 8 below.

Let $\id{Sum}(x,F(x,\by),P(x,\by),C(\by))$ be the result of invoking the solver
its inputs, and $\alpha$, $\beta$ stand for any expression.
The following cases are applied in order:

\paragraph{Case 0} if $C(\by)$ is unsatisfiable, return any expression (say, $0$).

\paragraph{Case 1}  if any literals in $F(x,\by)$ are trivially contradictory,
such as $\alpha \neq \alpha$, $\alpha < \alpha$, $\alpha \neq \beta$ for $\alpha$ and $\beta$ two distinct constants, return $0$.

\paragraph{Case 2} if any literals in $F(x,\by)$ are trivially true, (such as $\alpha = \alpha$ or $\alpha \geq \alpha$), or are redundant due to being identical to a previous literal,
return $\id{Sum}(x,F'(x,\by),P(x,\by),C(\by))$, for $F'(x,\by)$ equal to $F(x,\by)$ after removing such literals.

\paragraph{Case 3} if $F(x,\by)$ contains literal $x=\alpha$,
return $\id{Sum}(x,F'(x,\by),P(x,\by),C(\by))$,
for $F'(x,\by)$ equal to $F(x,\by)$ after replacing every \emph{other} occurrence of $x$ with $\alpha$.

\paragraph{Case 4} if any literal $L$ in $F(x,\by)$ does not involve $x$,
return the expression
\[\ifte{L}{\id{Sum}(x,F'(x,\by),P(x,\by),C(\by) \wedge L)}{0},\] for $F'(x,\by)$ equal to $F(x,\by)$ after removing $L$.

\paragraph{Case 5} if $F(x,\by)$ contains only literal $x=\alpha$, return $P(\alpha,\by)$.

\paragraph{Case 6} if $F(x,\by)$ contains literals $x\geq\alpha$ or $x<\beta$,
return $\id{Sum}(x,F'(x,\by),P(x,\by),C(\by))$,
for $F'(x,\by)$ equal to $F(x,\by)$ after replacing such literals by $x>\alpha-1$ and $x\leq\beta+1$, respectively.
This guarantees that all lower bounds for $x$ are strict, and all upper bounds are non-strict.

\paragraph{Case 7}
if $F(x,\by)$ contains literal $x > \alpha$ ($\alpha$ is a strict lower bound), and literal $x > \beta$ or literal $x \neq \beta$,
let literal $L$ be $\alpha < \beta$.
Otherwise, if $F(x,\by)$ contains literal $x \leq \alpha$ ($\alpha$ is a non-strict upper bound), and literal $x \geq \beta$ or literal $x \neq \beta$,
let literal $L$ be $\beta \leq \alpha$.
Otherwise, if $F(x,\by)$ contains literal $x \neq \alpha$ and literal $x \neq \beta$, let $L$ be $\alpha = \beta$.
Otherwise, if $F(x,\by)$ contains literal $x > \alpha$ and literal $x \leq \beta$, let $L$ be $\alpha < \beta$.
Then, if $C(\by) \wedge L$ and $C(\by) \wedge \neg L$ are both satisfiable (that is, $C(\by)$ does not imply $\alpha = \beta$ either way),
return the expression
\begin{align*}
	&\iftt{L} \thentt{\id{Sum}(x,F(x,\by),P(x,\by),C(\by)\wedge L)}\\
	&\qquad   \elsett{\id{Sum}(x,F(x,\by),P(x,\by),C(\by)\wedge \neg L)}.
\end{align*}

\paragraph{Case 8} At this point, $F(x,\by)$ and $C(\by)$ jointly define a single strict lower bound $l$ and non-strict upper bound $u$ for $x$, and $\{\beta_1,\dots,\beta_k\}$ such that $x \neq \beta_i$ and $l < \beta_i \leq u$ for every $i \in \{1,\dots,k\}$. If $C(\by)$ implies $u - l < k$, return $0$. Otherwise, return $\id{FH}\bigl(\sum_{x : l < x \leq u} P(x,\by) \bigr) - P(\beta_1,\by) - \dots - P(\beta_k,\by)$, where $\id{FH}$ is an extended version of Faulhaber's formula
\cite{knuth93faulhaber}. The extension is presented in Appendix \ref{sec:faulhaber} and only involves simple algebraic manipulation. 
The fact that Faulhaber's formula can be used in time independent of $u-l$ renders the solver
complexity independent of the index's domain size.

\begin{theorem}
	\label{the:difference-arithmetic-solver}
	Given $x$, $F(x,\by)$, $P(x,\by)$, $C(\by)$, the solver computes $\id{Sum}(x,F(x,\by),P(x,\by),C(\by))$
	in time independent\footnote{Strictly speaking, the complexity is logarithmic in the domain size,
		if arbitrarily large numbers and infinite precision are employed,
		but constant for all practical purposes.} of the domain sizes of $x$ and \by, and
	\begin{multline*}
		\forall \by \; C(\by) \Rightarrow \\
		\id{Sum}(x,F(x,\by),P(x,\by),C(\by)) \; = \sum_{x:F(x,\by)} P(x,\by).
	\end{multline*}
\end{theorem}

\begin{proof}
	(Sketch) Cases 0-2 are trivial (Case 0, in particular, is based on the fact
	that any solution is correct if $C(\by)$ is false).
	
	Cases 3 and 4 cover cases in which $x$ is bounded to a value and successively
	eliminate all other literals until trivial Case 5 applies.
	The left lower box of Figure \ref{fig:sgdpllt-example} exemplifies this pattern.
	
	Case 6 and 7 gradually determine a single strict lower bound $l$ and non-strict
	upper bound $u$ for $x$, determine that $l < u$,
	as well as which expressions $\beta_i$ constrained to be distinct from $x$
	are within $l$ and $u$, and are distinct from each other.
	This provides the necessary information for Case 8 to use Faulhaber's formula
	and determine a solution.
	The right lower box of Figure \ref{fig:sgdpllt-example} exemplifies this pattern.
\end{proof}

\section{Computing Faulhaber's extension $\id{FH}$}
\label{sec:faulhaber}

We now proceed to explain how $FH$ can computed the sum

\begin{align*}
	\sum_{x : l < x \leq u} t_0 + t_1 x + \dots + t_n x^n
\end{align*}
where $x$ is an integer index and
$t_i$ are monomials, possibly including numeric constants and powers of free variables.

\define{Faulhaber's formula} \cite{knuth93faulhaber} solves the simpler sum of powers problem $\sum_{k=1}^n k^p$:
\begin{align*}
	\sum_{k=1}^n k^p = {\frac{1}{p+1}} \sum_{j=0}^p (-1)^j{p+1 \choose j} B_j n^{p+1-j},
\end{align*}
where $B_j$ is a \emph{Bernoulli number} defined as
\begin{align*}
	B_j &= 1 - \sum_{k=0}^{j-1}\binom jk\frac{B_k}{j-k+1} \\
	B_0 &= 1.
\end{align*}

%Let us denote a polynomial $t_0 + t_1 x + \dots + t_n x^n$ by $\psi(t,x,n)$.

The original problem can be reduced to a sum of powers in the following manner,
where $t, r, s, v, w$ are families of monomials (possibly including numeric constants) in the free variables:
\begin{align*}
	& \sum_{x : l < x \leq u} t_0 + t_1 x + \dots + t_n x^n \\
	& = \sum_{i=0}^n \quad \sum_{x : l < x \leq u} t_i x^i \\
	& = \sum_{i=0}^n \sum_{x=1}^{u-l} t_i (x+l)^i \\
	& = \sum_{i=0}^n \sum_{x=1}^{u-l} t_i \sum_{q=0}^i r_q x^q \quad\text{(by expanding the binomial)}\\
	& = \sum_{i=0}^n \sum_{x=1}^{u-l} \sum_{q=0}^i t_i r_q x^q \\
	& = \sum_{i=0}^n \sum_{q=0}^i t_i r_q \sum_{x=1}^{u-l} x^q \;\text{(inverting sums to apply Faulhaber's)}\\
	& = \sum_{i=0}^n \sum_{q=0}^i {t_i r_q \over q+1} \sum_{j=0}^q (-1)^j{q+1 \choose j} B_j (u-l)^{q+1-j} \\
	& = \sum_{i=0}^n \sum_{q=0}^i \sum_{j=0}^q s_{i,q,j} (u-l)^{q+1-j} \\
	& = \sum_{i=0}^n \sum_{q=0}^i \sum_{j=0}^q s_{i,q,j} \sum_{l=1}^{q+1} v_l \quad\text{(by expanding the binomial)}\\
	& = \sum_{i=0}^n \sum_{q=0}^i \sum_{j=0}^q \sum_{l=1}^{q+1} s_{i,q,j}v_l\\
	& = w_0 + w_1 + \dots + w_{n'} \quad\text{(since $n$ is a known constant)}
\end{align*}
where $n'$ is function of $n$ in $O(n^4)$ (the time complexity for computing Bernoulli numbers up to $B_n$ is in $O(n^2)$).

Because the time and space complexity of the above computation depends on the initial
degree $n$ and the degrees of free variables in the monomials,
it is important to understand how these degrees are affected.
Let $d_l$ be the initial degree of the variable present in $l$ in $t$ monomials.
Its degree is up to $n$ in $r$ monomials (because of the binomial expansion with $i$ being up to $n$), and thus up to $d_l + n$ in $s$ monomials (because of the multiplication of $t_i$ and $r_q$).
The variable has degree up to $n + 1$ in monomials $v$, with degree up to $d_l + 2n + 1$
in the final polynomial.
The variable in $u$ keeps its initial degree $d_u$ until it is increased by up to $n + 1$ in $v$,
with final degree up to $d_u + n + 1$.
The remaining variables keep their original degrees.
This means that degrees grow only linearly over multiple applications of the above.
This combines with the $O(n^4)$ per-step complexity to a $O(n^5)$ overall complexity
for $n$ the maximum initial degree for any variable.
Note how this time complexity is constant in $x$'s domain size.

\bibliographystyle{named}
\bibliography{StarAI-16}

\end{document}